\documentclass{article} % For LaTeX2e
\usepackage{iclr2023_conference_tinypaper,times}

\usepackage{amsmath,amsfonts,bm}

\def\eqref#1{equation~\ref{#1}}
\def\1{\bm{1}}

\DeclareMathAlphabet{\mathsfit}{\encodingdefault}{\sfdefault}{m}{sl}
\SetMathAlphabet{\mathsfit}{bold}{\encodingdefault}{\sfdefault}{bx}{n}

\DeclareMathOperator*{\argmin}{arg\,min}

\usepackage{hyperref}
\usepackage{url}

\usepackage{graphicx}
\usepackage{amsthm,amssymb,amscd, color}
\usepackage[ruled,vlined,noresetcount]{algorithm2e}
\usepackage{caption}
\usepackage{subcaption}
\usepackage[makeroom]{cancel}

\newtheorem{lemma}{Lemma}

\theoremstyle{remark}

\hypersetup{colorlinks = true, urlcolor = blue, linkcolor = blue, citecolor=blue}

\title{Q-learning as a monotone scheme}

\author{Lingyi Yang\\ %\thanks{ Use footnote for providing further information
Mathematical Institute\\
University of Oxford\\
\texttt{yangl@maths.ox.ac.uk} \\
}

\iclrfinalcopy 
\begin{document}

\maketitle

\begin{abstract}
Stability issues with reinforcement learning methods persist.
To better understand some of these stability and convergence issues involving deep reinforcement learning methods, we examine a simple linear quadratic example.
We interpret the convergence criterion of exact Q-learning in the sense of a monotone scheme and discuss consequences of function approximation on monotonicity properties.
\end{abstract}

\section{Introduction}
\citet{sutton_reinforcement_2018} coined the combination of bootstrapping, off-policy learning, and function approximations as ``the deadly triad'' due to stability issues.
Let \(x\) denote the state, and \(u\) the action. A simple system for analysis is the linear quadratic (LQ) regulator where the dynamics is linear and the objective function is quadratic
\[x_{t+1} = A x_t + B u_t, \quad \min \bigg\{ \sum_t x_t^\intercal Qx_t + u_t^\intercal P u_t \bigg\}.\]
\citet{recht_tour_2019} observed unstable behaviour on this system using a vanilla policy gradient method.
\citet{agarwal_theory_2021} showed policy gradients converge when the value function is increasing/monotone pointwise.
We know that in general, the monotonicity of a numerical method can have a large impact on its convergence \citep{godunov_1959_finite, crandall_1980_monotone}.
We say an explicit numerical method with updates \(u^{n+1}_i = S(\{u^n_{j}\}_{j\in \mathcal{I}}), \:S: \mathbb{R}^k \to \mathbb{R}\)
is monotone if the operator \(S\) is monotone, i.e. if \(u \geq v \Rightarrow Su \geq Sv \).
Linear schemes are monotone only if all coefficients of \(u^n_j\) are non-negative (see the discussion on positive coefficient discretisation by \citet{forsyth_numerical_2007}).
\citet{barles_convergence_1991} proved that a stable, consistent, and monotone scheme converges (as the mesh size tends to zero) to the viscosity solution \citep{crandall_viscosity_1983, crandall_users_1992}.
We look at the impact of monotonicity in the context of continuous LQ problems, and the implications for Q-learning if we view it as a discretised numerical method.

\section{1D Deterministic Linear Quadratic Problem}

Consider a continuous, infinite-horizon, discounted control problem with linear state dynamics
\begin{equation}\nonumber
\frac{dX_t^{x,u}}{dt} = b(X_t^{x,u}, u_t), \quad  b(X_t^{x,u}, u_t) = \alpha X_t^{x,u} +u_t, \quad X^{x,u}(0) = x,
\end{equation}
where \(\{X^{x,u}\}_t\) is the process starting at \(x\) and following the policy \(u\).
Let the cost function be
\(J(u_s; X_s) = \int_0^\infty e^{-\beta s} f(X^{x,u}_s,u_s)ds,\)
where
\(f(X^{x,u}_t,u_t) = (X_t^{x,u})^2+u_t^2\), and define the value function \(V(x) = \inf_u J(u; x)\).
The Hamilton--Jacobi--Bellman (HJB) equation is 
\begin{equation}\label{eqn:hj1}
-\beta V(x)+\inf_u \big\{ \partial_x V(x) \cdot b(x,u) + f(x,u) \big\} = 0.
\end{equation}
To solve the HJB (\ref{eqn:hj1}) numerically, we can rearrange the formula
to get a fixed point problem
\begin{equation}\label{eqn:scheme}
V^{n+1} = \frac{\beta+\gamma}{\gamma} V^n + \frac{1}{\gamma} \min_u  \bigg\{ \partial_x V^n \cdot (\alpha x +u) + x^2 +u^2 \bigg\}.
\end{equation}
If we choose a finite difference scheme for the derivative depending on the sign of \(\alpha x_i + u\), then we can obtain an upwind method that ensures monotonicity.
This subtle variation of finite difference schemes for different parts of the domain
has an enormous impact on the stability in estimating value function and policies through value/policy iteration. A brief introduction to the continuous control set-up and the stability analysis of (\ref{eqn:scheme}) can be found in Appendix~\ref{app:1dctsprob}.

\subsection{Q-Learning}
Q-learning arises as a fixed point iteration to the Bellman optimality equation in discrete time
\begin{align*}
Q^{n+1}(x_t,u_t) 
=& (1-\alpha)Q^n(x_t,u_t) + \alpha\bigg[f(x_t,u_t) + \gamma \min_{\hat{u}}Q^n(x_{t+1}, \hat{u})\bigg],
\end{align*}
where \(f(x,u)\) denotes the instantaneous reward function (Appendix~\ref{app:1ddisprob}).
We see the coefficients are non-negative for \(0\leq \alpha \leq 1\); within this range, the update step is monotone. This aligns with the usual range for the step size \(\alpha\) in Q-learning. We explore larger values outside this range analogous to over-relaxation in optimisation \citep{saad2003iterative}. 
In our experiments, Q-learning is seen to be stable against the theoretical limits for a deterministic LQ problem when \(0\leq \alpha \leq 1\) (Figure~\ref{fig:0.8}). We converge to the correct value function, and the differences in policy are due to discretisation error. Since monotonicity is a sufficient condition for convergence, having \(\alpha\) outside of this range does not necessitate the method breaking. In Figure~\ref{fig:1.3} we see that we still converge to the theoretical values for \(\alpha = 1.3\). Instability occurs when \(\alpha\) is sufficiently large, \(\alpha = 1.8\) will do (Figure~\ref{fig:1.8}). 

\begin{figure*}[ht!]
    \centering
    \begin{subfigure}[t]{0.5\textwidth}
        \centering
        \includegraphics[height=1.7in]{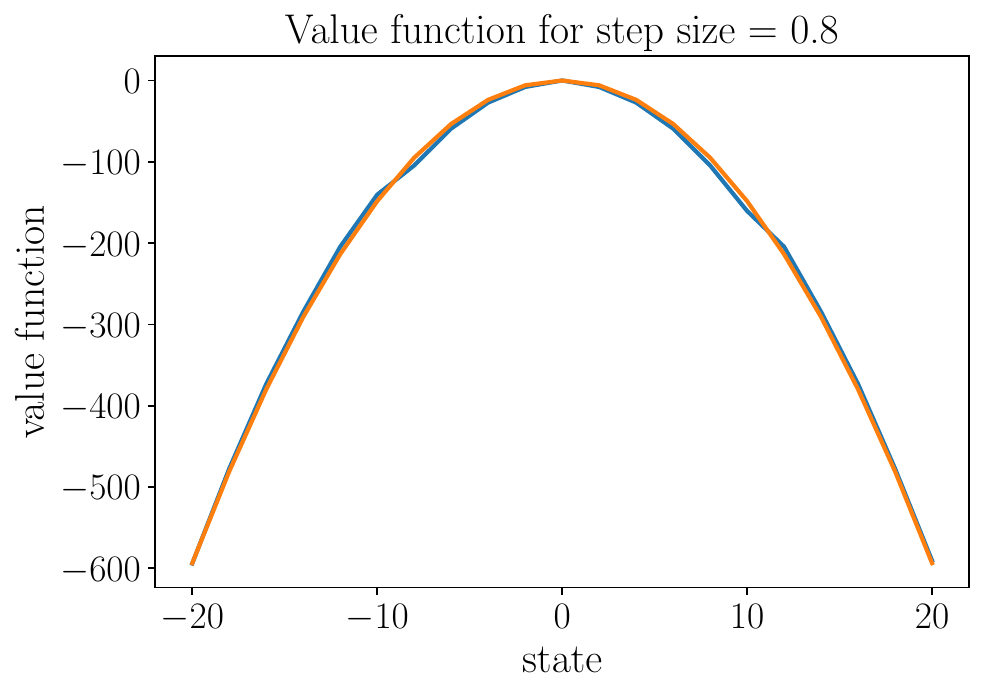}
        %\caption{}
    \end{subfigure}%
    %~ 
    \begin{subfigure}[t]{0.5\textwidth}
        \centering
        \includegraphics[height=1.7in]{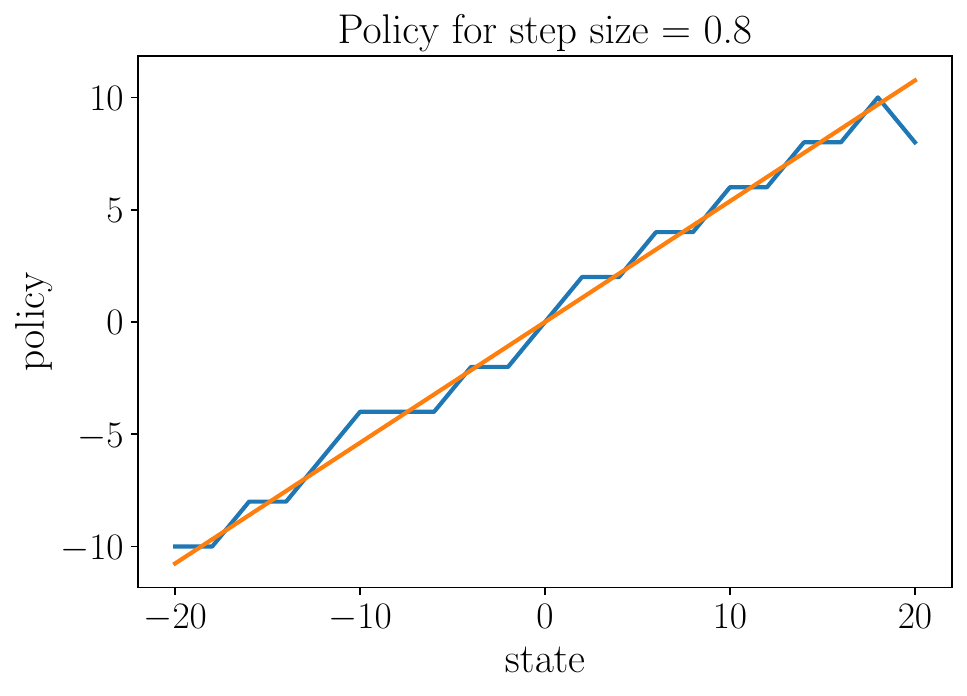}
        %\caption{L}
    \end{subfigure}
    \caption{Q-learning: learnt value function and policy (blue) against theoretical (orange) for $\alpha = 0.8$}\label{fig:0.8}
\end{figure*}
\begin{figure*}[ht!]
    \centering
    \begin{subfigure}[t]{0.5\textwidth}
        \centering
        \includegraphics[height=1.7in]{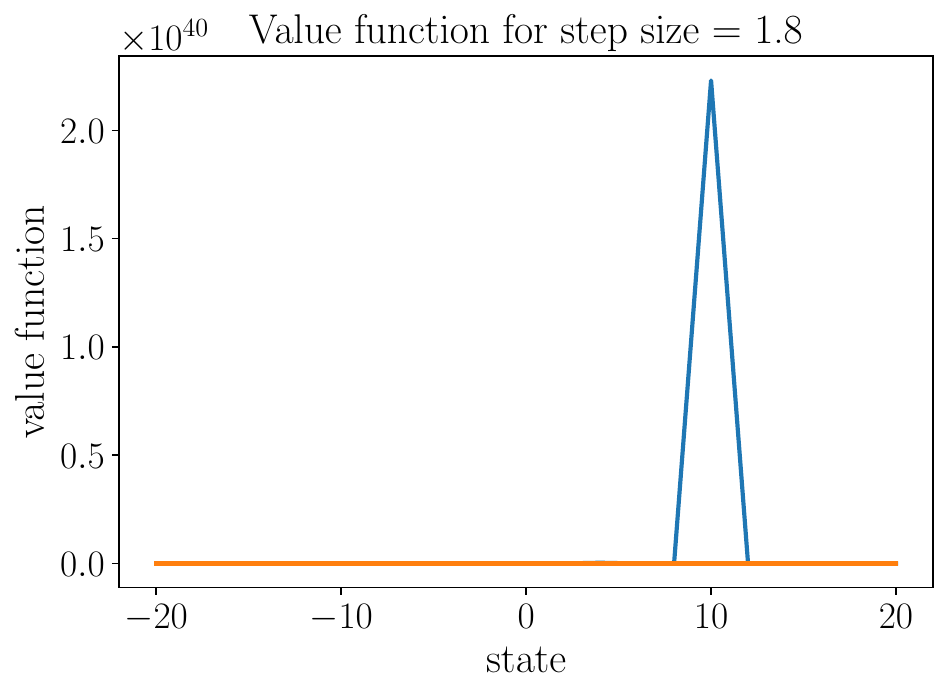}
        %\caption{}
    \end{subfigure}%
    %~ 
    \begin{subfigure}[t]{0.5\textwidth}
        \centering        \includegraphics[height=1.7in]{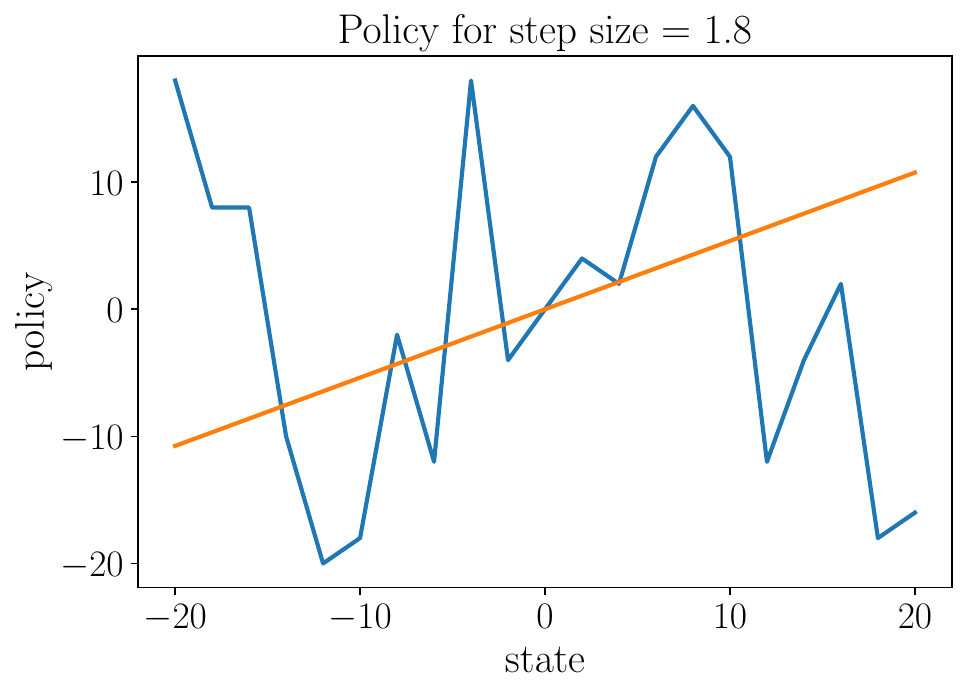}
        %\caption{}
    \end{subfigure}
    \caption{Q-learning: learnt value function and policy (blue) against theoretical (orange) for $\alpha = 1.8$} \label{fig:1.8}
\end{figure*}

Ensuring monotonicity with a function approximator 
is non-trivial. 
In the LQ case, note 
\(Q^*(x_t, u) = f(x_t, u) + V^*(x_{t+1})\)
and \(f\) and \(V\) can be expressed as a quadratic function in \(x\) and \(u\) (Appendix~\ref{app:1dctsprob}). Therefore
a linear function approximator for \(Q(x,u)\) with features of terms up to quadratic powers in \(x\) and \(u\) will be a suitable function class.
To be precise, 
\(\tilde{Q}(x, u, w) = X(x, u)^\intercal w\),
where \(X(x,u)\) are the features that we extract from our state-action pair, and \(w\) are the weights. Then (Appendix~\ref{app:1ddisprob})
\begin{align*}
\tilde{Q}^{n+1}(x, u,w_{n+1})
 =& \:\alpha_n f(x,u) X^\intercal(x, u)X(x, u) + (1-\alpha_n X^\intercal(x, u) X(x, u)) \tilde{Q}^n(x,u,w_n) \nonumber\\&+ \alpha_n X^\intercal(x, u) X(x, u) \max_{\tilde{u}} \tilde{Q}^n(x',\tilde{u},w_n). %\label{eqn:lq_update}
\end{align*}
To ensure monotonicity, the features \(X(x,u)\) need to be bounded and step sizes are sufficiently small such that \(\alpha_n X^\intercal(x, u) X(x, u) <1\), i.e. \(\alpha_n(x,u) < 1/ (X^\intercal(x, u) X(x, u)) \). This condition is dependent on the state and action, so a potential issue is that we may not sufficiently explore the state space (e.g. if for large values of \(x\), the action \(u\) is also large then \(\alpha\) needs to be very small). Even a simple, linear function approximator can disrupt monotonicity causing instability so violations in the nonlinear case (neural networks) 
may explain the stability issues we observe in practice.

\subsubsection*{Acknowledgements}
The author would like to thank Prof Samuel N. Cohen and Dr Jaroslav Fowkes for their support and feedback. This work was supported by the EPSRC [EP/L015803/1].

\subsubsection*{URM Statement}
The authors acknowledge that at least one key author of this work meets the URM criteria of ICLR 2024 Tiny Papers Track.

\bibliography{refs}
\bibliographystyle{iclr2023_conference_tinypaper}

\newpage

\appendix

\section{1D continuous LQ problem} \label{app:1dctsprob}

Consider the typical deterministic, infinite-horizon, discounted control problem with linear state dynamics
\begin{equation}\label{eqn:dx}
\frac{dX_t^{x,u}}{dt} = b(X_t^{x,u}, u_t), \quad  b(X_t^{x,u}, u_t) = \alpha X_t^{x,u} +u_t, \quad X^{x,u}(0) = x,
\end{equation}
where \(\{X^{x,u}\}_t\) is the process starting at \(x\) and following the policy \(u\) thereafter.
Let the cost function be given by
\begin{equation}\label{eqn:cost}
J(u_s; X_s) = \int_0^\infty e^{-\beta s} f(X^{x,u}_s,u_s)ds,
\end{equation}
where
\begin{equation} \nonumber
f(X^{x,u}_t,u_t) = (X_t^{x,u})^2+u_t^2.
\end{equation}
Define the value function 
\begin{equation}\label{eqn:value}
V(x) = \inf_u J(u; x).
\end{equation}

\begin{lemma}
The dynamic programming principle gives us
\begin{equation}\nonumber
V(x) = \inf_u  \bigg \{\int_0^h e^{-\beta s} \big((X_s^{x,u})^2+u_s^2\big)ds + e^{-\beta h} V(X_h^{x,u})\bigg \}.
\end{equation}
\end{lemma}

\begin{proof}
\begin{align*}
V(x) &= \inf_u \int_0^\infty e^{-\beta s} f(X^{x,u}_s,u_s)ds \\
&= \inf_u \int_0^h e^{-\beta s} f(X^{x,u}_s,u_s)ds + \inf_u \int_h^\infty e^{-\beta s} f(X^{x,u}_s,u_s)ds \\
&=\inf_u \int_0^h e^{-\beta s} f(X^{x,u}_s,u_s)ds + \inf_{\tilde{u}} \int_0^\infty e^{-\beta( t+h)} f(X^{X_h^{x,u},\tilde{u}}_t,\tilde{u}_t)dt  \\
& = \inf_u \int_0^h e^{-\beta s} f(X^{x,u}_s,u_s)ds + e^{-\beta h}\inf_{\tilde{u}} \int_0^\infty e^{-\beta t} f(X^{X_h^{x,u},\tilde{u}}_t,\tilde{u}_t)dt \\
&= \inf_u \int_0^h e^{-\beta s} f(X^{x,u}_s,u_s)ds + e^{-\beta h} V(X_h^{x,u}).
\end{align*}
\end{proof}
See \citep{krylov_controlled_1980} for a rigorous formulation of this problem detailing the set of admissible controls and growth conditions assumed.
From this form of the value function, we can derive the Hamilton--Jacobi--Bellman (HJB) equation. 

\begin{lemma}
The HJB equation of the system given by (\ref{eqn:dx}) and (\ref{eqn:cost}) is 
\begin{equation}\label{eqn:hj}
-\beta V(x)+\inf_u \big\{ \partial_x V(x) \cdot b(x,u) + f(x,u) \big\} = 0.
\end{equation}
\end{lemma}

\begin{proof}
Let us first consider a constant control \(\bar{u}\). By the definition of the value function, we must have
\begin{equation}\nonumber
V(x) \leq \int_0^h e^{-\beta s} f(X^{x,\bar{u}}_s,\bar{u})ds + e^{-\beta h} V(X_h^{x,\bar{u}}).
\end{equation}
Now by the chain rule we have
\begin{equation}\nonumber
V(X_h^{x,\bar{u}}) = V(x) + \int_0^h b(X_t^{x, \bar{u}}, \bar{u}) \partial_x V(X_t^{x,\bar{u}}) dt,
\end{equation}
therefore
\begin{equation}\nonumber
V(x) \leq \int_0^h e^{-\beta s} f(X^{x,\bar{u}}_s,\bar{u})ds + e^{-\beta h} \bigg( V(x) + \int_0^h b(X_t^{x, \bar{u}}, \bar{u}) \partial_x V(X_t^{x,\bar{u}}) dt\bigg).
\end{equation}
Upon rearranging, we have
\begin{equation}\nonumber
\frac{1-e^{-\beta h}} {h} V(x) \leq \frac{1}{h}\int_0^h e^{-\beta s} f(X^{x,\bar{u}}_s,\bar{u})ds + \frac{e^{-\beta h}}{h}  \int_0^h b(X_t^{x, \bar{u}}, \bar{u}) \partial_x V(X_t^{x,\bar{u}}) dt.
\end{equation}

We consider the limit as \(h\) tends to 0. By L'H\^{o}pital's rule
\begin{equation}\nonumber
\lim_{h\to 0} \frac{1-e^{-\beta h}}{h} = \lim_{h\to 0} \frac{\beta e^{-\beta h}}{1} = \beta,
\end{equation}
and by applying the Mean Value Theorem, we obtain
\begin{equation}\nonumber
\beta V(x) \leq f(x, \bar{u}) + b(x, \bar{u}) \partial_x V(x).
\end{equation}
Therefore, for an arbitrary constant cost \(\bar{u}\) we have
\begin{equation}\nonumber
-\beta V(x) + f(x, \bar{u}) + b(x, \bar{u}) \partial_x V(x) \geq 0,
\end{equation}
thus
\begin{equation}\nonumber
-\beta V(x) + \inf_u\big[f(x, {u}) + b(x, {u}) \partial_x V(x)\big] \geq 0.
\end{equation}
If we apply the above analysis with the optimal control, we will find that equality holds
\begin{equation}\nonumber
-\beta V(x) + \big[f(x, {u}^*) + b(x, {u}^*) \partial_x V(x)\big] = 0,
\end{equation}
and therefore
\begin{equation}\nonumber
-\beta V(x) + \inf_u\big[f(x, {u}) + b(x, {u}) \partial_x V(x)\big] = 0
\end{equation}
as required.
\end{proof}

We can find the optimal feedback control of (\ref{eqn:hj}).
 
\begin{lemma}
The optimal control is given by
\begin{equation}\nonumber
u^* = -\Gamma x.
\end{equation}
\end{lemma}
\begin{proof}
Let us propose the following ansatz for (\ref{eqn:hj})
\begin{equation}\nonumber
V(x) = \Gamma x^2 + 2\kappa x + \lambda.
\end{equation}
We then find the derivative with respect to \(x\), \(dV/dx = 2\Gamma x + 2 \kappa\)  and substitute into the Hamiltonian to get
\begin{equation}\nonumber
\partial_x V(x)\cdot b(x,u) + f(x,u) = Qx^2 + Ru^2 + (2\Gamma x + 2\kappa)(Ax+Bu).
\end{equation}
By taking the partial derivative w.r.t \(u\) and setting to zero for the stationary point, we have that 
\begin{equation}\nonumber
2 Ru + B(2\Gamma x + 2K) = 0
\end{equation}
and the optimal control is given by
\begin{equation}\nonumber
u^* = -\frac{B(\Gamma x+ \kappa)}{R}.
\end{equation}
Evaluate the Hamiltonian at the optimal control
\begin{align*}
\partial_x V(x)\cdot b(x,u) + f(x,u)|_{u^*} &= Qx^2+ R\bigg(-\frac{B(\Gamma x+ \kappa)}{R} \bigg)^2 + (2\Gamma x + 2 \kappa)\bigg(Ax - \frac{B^2(\Gamma x+ \kappa)}{R} \bigg)\\
&= Qx^2 + \cancel{\frac{B^2(\Gamma x + \kappa)^2}{R}} + 2 A \Gamma x^2 + 2 A \kappa x - \frac{\cancel{2}B^2(\Gamma x + \kappa)^2}{R}\\
& = \bigg(Q+2A\Gamma  - \frac{B^2\Gamma ^2}{R}\bigg) + \bigg(2A\kappa - \frac{2B^2\Gamma \kappa}{R}\bigg) x - \frac{B^2\kappa^2}{R}.
\end{align*}

Substitute this into the HJB 
\begin{equation}\nonumber
\beta(\Gamma x^2 + \kappa x + \lambda) - \Bigg(\bigg(Q+2A\Gamma  - \frac{B^2\Gamma ^2}{R}\bigg) + \bigg(2A\kappa - \frac{2B^2\Gamma \kappa}{R}\bigg) x - \frac{B^2\kappa^2}{R} \Bigg) = 0
\end{equation}
and set each coefficient to zero
\begin{align}
0 &= \frac{B^2\Gamma ^2}{R} + \Gamma(\beta -2A) - Q, \label{eqn:Ox2} \\
0 &= 2\kappa(\beta + \frac{B^2\Gamma}{R}-A),\label{eqn:Ox}\\
0 &= \beta \lambda + \frac{B^2 \kappa^2}{R}\label{eqn:O1}.
\end{align}
From (\ref{eqn:Ox}), we see that either \(\kappa = 0\) or \(\Gamma = \frac{R(A-\beta)}{B^2}\).
The latter case would not generally satisfy (\ref{eqn:Ox2}). Thus we must have \(\kappa = 0\). Substituting this into (\ref{eqn:O1}), we also get that \(\lambda = 0\). Thus the only non-zero coefficient of \(V(x)\) is \(\Gamma\), which satisfies the quadratic (\ref{eqn:Ox2}). We choose the root that ensures we have a stable solution (typically positive definite).

For our problem (\ref{eqn:dx}) we have that \(A = \alpha\), \(B = R = Q =1\), hence \(\Gamma\) must satisfy
\begin{equation}
\Gamma^2 + (\beta -2 \alpha)\Gamma-1 = 0,
\end{equation}
which has two roots \(\Gamma^+\) and \(\Gamma^-\). We choose the root that would result in a positive eigenvalue. % !!!
Our optimal control is then \(u^* = -\Gamma x\).
\end{proof}

If we want to solve the HJB (\ref{eqn:hj}) numerically, we can rearrange the formula
to get a fixed point method
\begin{equation}
V^{n+1} = \frac{\beta+\gamma}{\gamma} V^n + \frac{1}{\gamma} \min_u  \bigg\{ \partial_x V^n \cdot (\alpha x +u) + x^2 +u^2 \bigg\}.
\end{equation}

If we na\"{i}vely tried to solve this directly by approximating the derivative with finite differences, for example, central difference, then we obtain the following numerical method
\begin{equation}
V_i^{n+1} = \frac{\beta+\gamma}{\gamma} V_i^n + \frac{1}{\gamma} \min_u  \bigg\{ \frac{V^n_{i+1}-V^n_{i-1}}{2\Delta x} \cdot (\alpha x_i +u) + x_i^2 +u^2 \bigg\},
\end{equation}
where \(V_i = V(x_i)\) and \(x_i = x_{i-1}+\Delta x\), and we can take forward and backward differencing at the boundary.

This method is not monotone in general as the coefficients of \(V_{i+1}^n\) and \(V_{i-1}^n\) are the opposite signs to each other.

To obtain a monotone method we look at an upwind scheme. Let us approximate the derivative w.r.t. \(x\) with forward differences or backward differences depending on the sign of \(\alpha x +u\)
\begin{equation} \nonumber
\partial_x V \approx \frac{V_{i+1}-V_{i}}{\Delta x} \quad \text{if }\:\: \alpha x + u>0
\end{equation}
\begin{equation} \nonumber
\partial_x V \approx \frac{V_{i}-V_{i-1}}{\Delta x} \quad \text{if }\:\: \alpha x + u<0.
\end{equation}

For \(\alpha x + u>0\),
\begin{align*}
V_i^{n+1} &= \frac{\beta+\gamma}{\gamma} V_i^n + \frac{1}{\gamma} \min_u  \bigg\{ \frac{V^n_{i+1}-V^n_{i}}{\Delta x} (\alpha x_i +u) + x_i^2 +u^2 \bigg\}\\
& = \bigg(\frac{\beta+\gamma}{\gamma} -\frac{\alpha x_i +u^*}{\gamma \Delta x} \bigg)V_i^n + \frac{\alpha x_i +u^*}{\gamma \Delta x} V^n_{i+1}+ \frac{x_i^2 +(u^*)^2}{\gamma},
\end{align*}
where \(u^*\) is the argmin of the Hamiltonian.
Now the coefficient of \(V_{i+1}^n\) is positive, but the coefficient of \(V_{i}^n\) will only be positive if we have
\begin{equation}
\Delta x > \frac{\alpha x_i + u^*}{\beta + \gamma}.
\end{equation}
Hence taking a large value for \(\gamma\) enables us to use finer meshes.

Similarly when \(\alpha x + u<0\),
\begin{align*}
V_i^{n+1} &= \frac{\beta+\gamma}{\gamma} V_i^n + \frac{1}{\gamma} \min_u  \bigg\{ \frac{V^n_{i}-V^n_{i-1}}{\Delta x} (\alpha x_i +u) + x_i^2 +u^2 \bigg\}\\
& = \bigg(\frac{\beta+\gamma}{\gamma} +\frac{\alpha x_i +u^*}{\gamma \Delta x} \bigg)V_i^n - \frac{\alpha x_i +u^*}{\gamma \Delta x} V^n_{i-1}+ \frac{x_i^2 +(u^*)^2}{\gamma}.
\end{align*}
Since \(\alpha x + u<0\), the coefficient of \(V_{i-1}^n\) is positive, but the coefficient of \(V_{i}^n\) will only be positive if we have
\begin{equation}
\Delta x > -\frac{\alpha x_i + u^*}{\beta + \gamma}.
\end{equation}

Note that by updating our value function as
\begin{equation} \nonumber
V^{n+1} = \frac{\beta+\gamma}{\gamma} V^n + \frac{1}{\gamma}\min_u\{\partial_x V^{n}\cdot (ax+u) + (x^2+u^2)\},
\end{equation}
this is precisely the value iteration updates. We only do one cycle of policy evaluation before choosing a new policy by taking \(u\) to be the argmin of the Hamiltonian (policy improvement). A monotone numerical scheme for the LQ problem with value iteration updates is described in Algorithm~\ref{algo:lq_value}.

\begin{algorithm}[H]
\SetAlgoLined
 1. Initialisation \\
 Discretise the state and action spaces \\
Parameters: a small threshold $\theta > 0$ determining accuracy of estimation (convergence criterion), a maximum iteration count \(N\)\\
Initialize $V(x)$ as zeros\\ 
Initialize $\Delta > \theta$. \\

 2. Iteration \\
 \Repeat{$\Delta<\theta$ or $N$}{
  
  \For{each $x\in \mathcal{X}$}{
  $\hat{V}(x)\leftarrow V(x)$,\\
  $ u^* \leftarrow$ argmin Hamiltonian at \(x\)\\
  $V(x) \leftarrow \frac{\beta +\gamma}{\gamma} \hat{V}(x) + \frac{1}{\gamma}H(x, u^*, \hat{V} )$,\\
  where $H(x, u^*, \hat{V} )$ is the Hamiltonian evaluated at \(x\), with control \(u^*\) and using the suitable differencing for a monotone scheme.\\
  }
  $\Delta \leftarrow \max|\hat{V}-V|$.\\
  }
 
 \caption{Value Iteration for LQ}\label{algo:lq_value}
\end{algorithm}

For a policy iteration-like update, we need to have a fixed policy that we evaluate the value function on until convergence before we make a policy improvement step. The pseudocode is given in Algorithm~\ref{algo:lq_policy} 

\begin{algorithm}[h]
\SetAlgoLined
 1. Initialisation \\
 Discretise the state and action spaces \\
Parameters: two small thresholds $\theta_v > 0$ and $\theta_u > 0$ determining accuracy of estimation (convergence criterion), maximum iteration counts \(N_v, \:\: N_u\)\\
Initialize \(u(x)\in \mathcal{U}\) arbitrarily for all $x \in X $ (let us say equal to 1). \\ 
Initialize $\Delta > \theta$. \\

 2. Policy Evaluation \\
 \Repeat{$\Delta<\theta_v$}{
  
  \For{each $x\in \mathcal{X}$}{
  $\hat{V}(x)\leftarrow V(x)$,\\
  $V(x) \leftarrow \frac{\beta +\gamma}{\gamma} \hat{V}(x) + \frac{1}{\gamma}H(x, u, \hat{V} )$,\\
  where $H(x, u, \hat{V} )$ is the Hamiltonian evaluated at \(x\), with the current control \(u\) and using the suitable differencing for a monotone scheme.\\
  }
  $\Delta \leftarrow \max|\hat{V}-V|$.\\
  }
  3. Policy Improvement \\
  \For{each $x\in \mathcal{X}$}{
  \(\hat{u}(x) \leftarrow u(x)\)\\
  \(u(x) \leftarrow\)  argmin of Hamiltonian at \(x\) \\

  }
  If $ \max|\hat{u}-u|<\theta_u$, then stop and return $V\approx V^*$ and $u\approx u^*$; else go back to Step 2.
 \caption{Policy Iteration for LQ}\label{algo:lq_policy}
\end{algorithm}

Let us be more precise on finding the control. To recap, for the value iteration, we are updating at each iteration with the rule
\begin{equation} \nonumber
V^{n+1} = -\frac{\beta-\gamma}{\gamma} V^n + \frac{1}{\gamma}\min_u\Big\{\partial_x V^{n}\cdot (\alpha x+u) + (x^2+u^2)\Big\}.
\end{equation}
In order to obtain a monotone scheme, we must apply forward differencing on \(\partial_x V^n\) if \(\alpha x+u>0\) or backward differencing otherwise. 
However, for value iteration, we are also minimising over all \(u\), which means that there are a few cases we can fall into. Let us consider finding the correct action for each discretised state \(x_i\).

We have the regions $ R_1 = \{u: \alpha x_i +u \geq 0\} $ and  $ R_1 = \{u: \alpha x_i +u <0\} $.
Let 
\begin{equation}\nonumber
H^n(u) = \begin{cases} h_1^n(u) \quad \text{if  } u\in R_1 \\
h_2^n(u) \quad \text{if  } u\in R_2
\end{cases}
\end{equation}
where
\begin{equation}\nonumber
h_1^n = \frac{V^n_{i+1}-V^n_{i}}{\Delta x} (\alpha x_i +u) + x_i^2 +u^2,
\end{equation}
and
\begin{equation}\nonumber
h_2^n = \frac{V^n_{i}-V^n_{i-1}}{\Delta x} (\alpha x_i +u) + x_i^2 +u^2.
\end{equation}
Our problem is now to find
\begin{equation}\nonumber
y = \min_u H^n(u).
\end{equation}
The smoothness of \(H(u)\) depends on the smoothness of the value function. In the case of the LQ problem, the value function is smooth, therefore \(H(u)\) is smooth except on the boundary of \(R_1\) and \(R_2\), which in this case is a linear boundary, \(u = -\alpha x_i\).

We can solve the above problem numerically, by finding 
\begin{equation}\nonumber
y^*_1 = \min_u h_1(u), \quad u\in R_1
\end{equation}
and
\begin{equation}\nonumber
y^*_2 = \min_u h_2(u), \quad u\in R_2
\end{equation}
then taking the minimum over these 2 values
\begin{equation}\nonumber
y = \min \{y_1,y_2\}
\end{equation}
to get the desired control for value iteration.

For policy iteration, when we are doing policy evaluation, as the policy is fixed, we just need to check the sign of \(\alpha x +u\). However, we need to also ensure that monotonicity is maintained in policy improvement. The policy improvement step can be described as
\begin{equation}\nonumber
u^* = \argmin_u H(u).
\end{equation}
In this case we can again solve the two minimisation problem separately to find \(y_1^*\) and \(y_2^*\) and choose \(u\) based on which of these have the lower value. 

In Figures~\ref{fig:downwind} and \ref{fig:upwind}, we see the result of applying the downwind iteration and upwind iteration respectively. 
We see clearly the instability arising in the downwind case, and how important it is to choose between forward and backward difference so that we ensure monotonicity.

\begin{figure*}[t!]
    \centering
    \begin{subfigure}[t]{0.5\textwidth}
        \centering
        \includegraphics[height=1.682in]{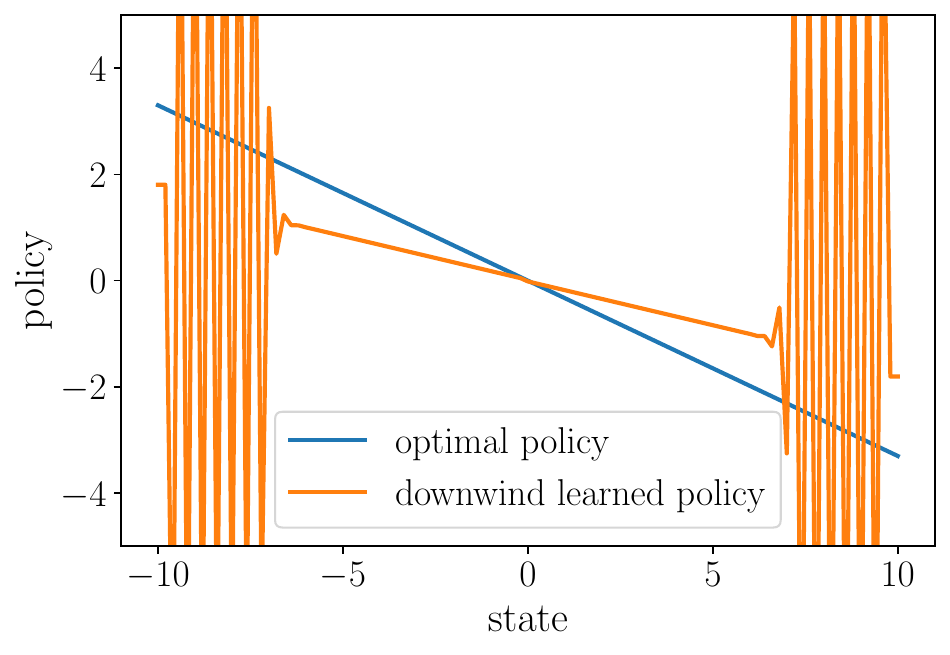}
        %\caption{}
    \end{subfigure}%
    %~ 
    \begin{subfigure}[t]{0.5\textwidth}
        \centering
        \includegraphics[height=1.682in]{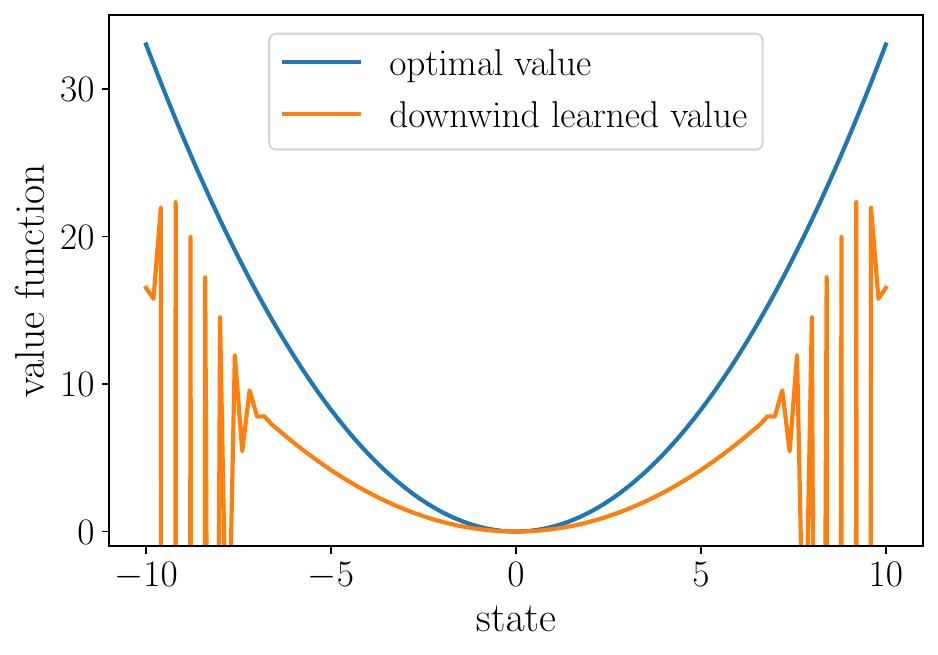}
        %\caption{}
    \end{subfigure}
    \caption{An intermediate policy and value function for a downwind method (the policy and value function have not converged yet). Instability forms and becomes amplified with further iterations.} \label{fig:downwind}
\end{figure*}

\begin{figure*}[t!]
    \centering
    \begin{subfigure}[t]{0.5\textwidth}
        \centering
        \includegraphics[height=1.682in]{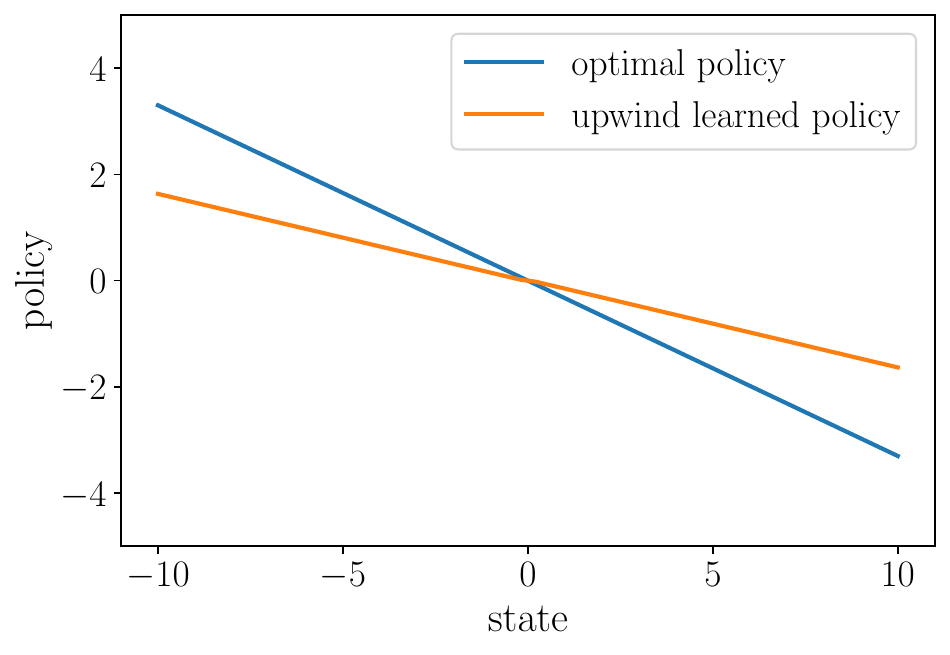}
        %\caption{}
    \end{subfigure}%
    %~ 
    \begin{subfigure}[t]{0.5\textwidth}
        \centering
        \includegraphics[height=1.682in]{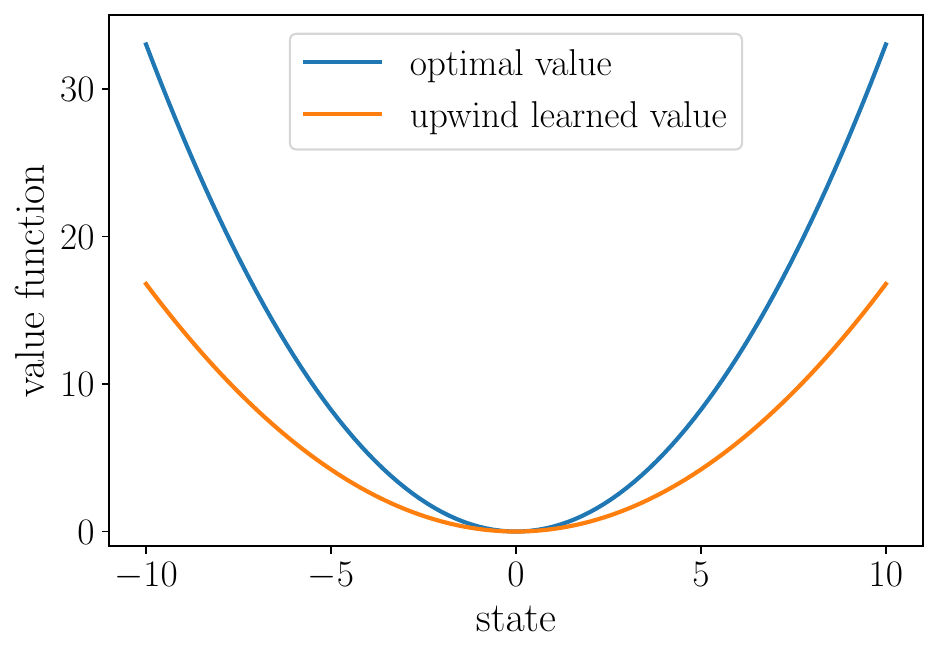}
        %\caption{}
    \end{subfigure}
    \caption{An intermediate policy and value function for an upwind method. Note that whilst the policy has not converged yet, there are no instabilities in this case.} \label{fig:upwind}
\end{figure*}

\section{Q-learning (discrete setting)}\label{app:1ddisprob}

Let the transition be \(x_{t+1} = b(x_t,u_t)\)
and let \(f(x,u)\) denote the reward function. 
The state value function under policy \(\pi\) is defined as
\begin{equation}\nonumber
V^\pi(x) = \sum_{t=0}^\infty \gamma^t f(x_t, \pi(x_t)), \quad x_0 = x.
\end{equation}

The Bellman optimality equations are
\begin{equation}\nonumber
V^*(x_t) = \max_{u_t} \Big[ f(x_t, u_t) + \gamma V^*(x_{t+1})|_{x_{t+1} = b(x_t,u_t)} \Big],
\end{equation}
\begin{align*}
Q^*(x_t, u_t) &= f(x_t, u_t) + \gamma V^*(x_{t+1})|_{x_{t+1} = b(x_t,u_t)} \\
&= f(x_t, u_t) + \gamma \min_{\hat{a}} Q(x_{t+1}, \hat{a})|_{x_{t+1} = b(x_t,u_t)}.
\end{align*}

Thus the first order condition is given by
\begin{equation}\nonumber
\frac{\partial Q^*(s_{t+1},\hat{a})}{\partial \hat{a}} = 0
\end{equation}
or
\begin{equation}\nonumber
\frac{\partial f}{\partial u_t} + \gamma \frac{\partial V^*}{\partial x_{t+1}}\frac{\partial x_{t+1}}{\partial u_t} = 0.
\end{equation}

Monotonicity gives a sufficient condition for stability but it is not necessary. We see that Q-learning is also stable for \(\alpha=1.3\) case, as seen in Figure~\ref{fig:1.3}.

\begin{figure*}[t!]
    \centering
    \begin{subfigure}[t]{0.5\textwidth}
        \centering
        \includegraphics[height=1.8in]{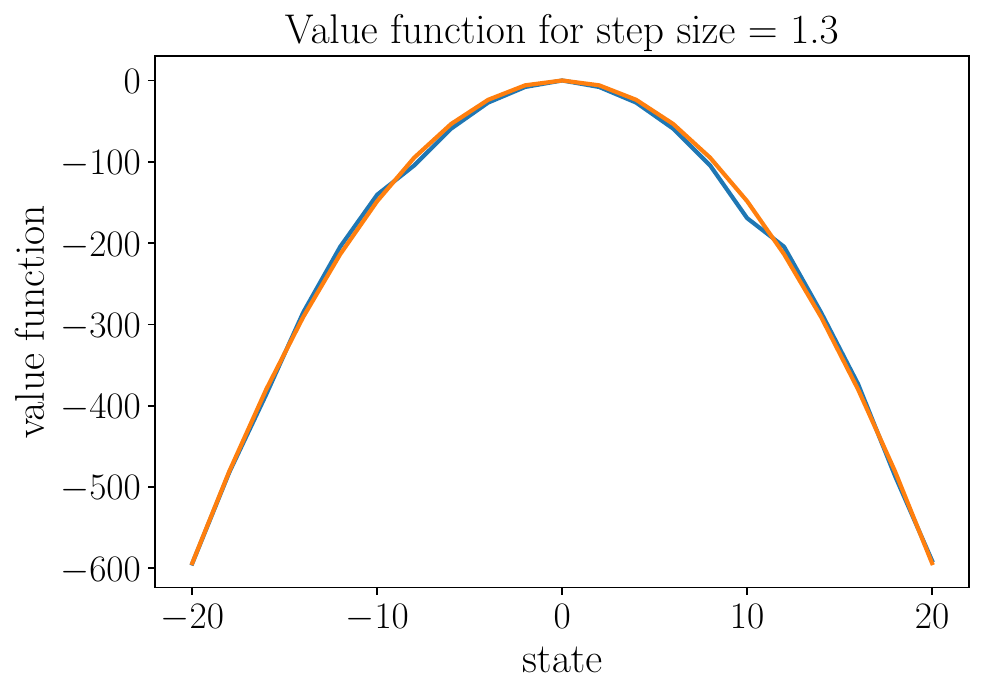}
        %\caption{}
    \end{subfigure}%
    %~ 
    \begin{subfigure}[t]{0.5\textwidth}
        \centering
        \includegraphics[height=1.8in]{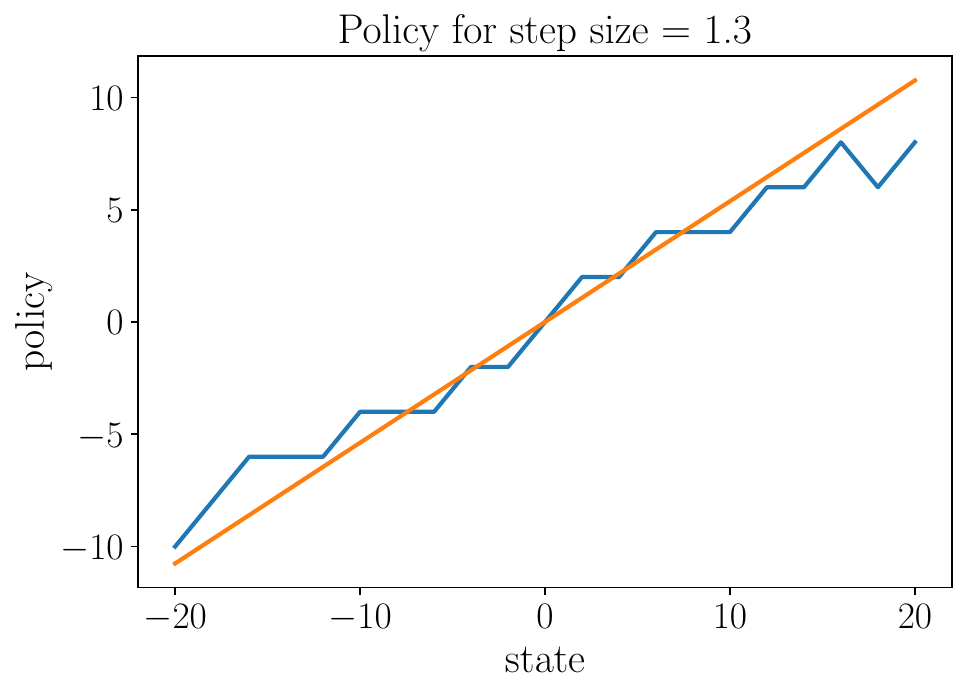}
        %\caption{}
    \end{subfigure}
    \caption{Learnt value function and policy (blue) against theoretical (orange) for $\alpha = 1.3$} \label{fig:1.3}
\end{figure*}

Note that, when we have a constant step size \(\alpha>1\), not only can we no longer guarantee monotonicity, but the square summability condition of the step size \(\alpha\) \citep{robbins_stochastic_1951} would also be violated. 

For `table-lookup' methods, as long as all states are updated infinitely often and step sizes satisfy square summability conditions \citep{robbins_stochastic_1951}, then we have convergence for policy evaluation.
This is not guaranteed when we have a general function approximator. 
The problem with having exact table representation is that they are slow to converge, and the number of states/state-action pairs suffer from the curse of dimensionality. Using a function approximator, we postulate that as long as the function approximation preserves monotonicity, then Q-learning should still converge.

When we are approximating \(Q(x, u)\) by \(\tilde{Q}(x, u, w) \), we want to minimise the expected error
\begin{equation}\nonumber
\frac{1}{2}E\big[(Q(x, u)-\tilde{Q}(x, u, w) )^2\big].
\end{equation}
We use a semi-gradient descent method and take update steps in the form of
\begin{equation} \nonumber
w_{n+1} = w_n + \alpha_n \big( Q(x, u) -\tilde{Q}(x, u, w_n) \big) \nabla_{w_n} \tilde{Q}(x, u, w_n).
\end{equation}
For reinforcement learning/control problems, we do not have access to the true value function \(Q(x,u)\) and therefore use an approximation in its place. An approximate method derived from Q-learning uses 
\begin{equation}\nonumber
Q(x,u) \approx f(x, u) + \gamma \max_{\bar{u}} \tilde{Q}(x', \bar{u}, w).
\end{equation}
Since
\begin{equation}\nonumber
Q^*(x_t, u) = f(x_t, u) + V^*(x_{t+1})
\end{equation}
and we know in the LQ case, both \(f\) and \(V\) can be expressed as a quadratic function in \(x\) and \(u\) (Appendix~\ref{app:1dctsprob}), a linear function approximator for \(Q(x,u)\) with features of terms up to quadratic powers in \(x\) and \(u\) will be a suitable function class for this approximation.

This linear approximator for \(Q(x,u)\) is represented as
\begin{equation} \label{eqn:linearapprox}
\tilde{Q}(x, u, w) = X(x, u)^\intercal w,
\end{equation}
where \(X(x,u)\) are the (e.g. quadratic) features that we extract from our state-action pair, and \(w\) are the weights. Then 
\begin{align}
\tilde{Q}^{n+1}(x, u,w_{n+1})
=& \:X^\intercal(x,u)\big(w_n + \alpha_n \big( Q(x, u) -\tilde{Q}^n(x, u, w_n) \big) X(x, u)\big) \nonumber \\
\approx &\:\tilde{Q}^n(x, u, w_n) \nonumber \\
&+ \alpha_n X^\intercal(x,u) \big( f(x,u) + \gamma \max_{\tilde{u}} \tilde{Q}^n(x', \tilde{u}, w_n) -\tilde{Q}^n(x, u, w_n) \big) X(x, u)\big) \nonumber\\
 =& \:\alpha_n f(x,u) X^\intercal(x, u)X(x, u) + (1-\alpha_n X^\intercal(x, u) X(x, u)) \tilde{Q}^n(x,u,w_n) \nonumber\\&+ \alpha_n X^\intercal(x, u) X(x, u) \max_{\tilde{u}} \tilde{Q}^n(x',\tilde{u},w_n). \label{eqn:lq_update}
\end{align}

Monotonicity implies that the action-value function \(\tilde{Q}^{n+1}(x, u,w_{n+1})\) must be non-decreasing with respect to the action-value function at the other points. Hence to ensure monotonicity we need to ensure that the features \(X(x,u)\) are bounded and we take sufficiently small step size such that \(\alpha_n X^\intercal(x, u) X(x, u) <1\), i.e. \(\alpha_n(x,u) < 1/ (X^\intercal(x, u) X(x, u)) \), the step size is dependent on the state and action. A potential problem is that we may not sufficiently explore the state space with a step size that is dependent on the state-action space (e.g. if for large values of \(x\), the action \(u\) is also large then \(\alpha\) needs to be very small).

\end{document}